\RequirePackage{fix-cm}
\documentclass[smallextended]{svjour3}       
\smartqed  
\usepackage{graphicx}
\usepackage{amsmath,amsfonts,amssymb}
\usepackage{mathtools}
\usepackage{listings}
\usepackage[sort]{natbib}
\makeatletter
\newsavebox\tboxa
\newsavebox\tboxb
\newlength\tdima

\newcommand*{\oversymb}{\mathpalette\@oversymb}

\newcommand*{\@oversymb}[2]{%
    \sbox{\tboxa}{$\m@th#1\mathrm{#2}$}%
    \setbox\tboxb\null%
    \ht\tboxb\ht\tboxa%
    \dp\tboxb\dp\tboxa%
    \wd\tboxb\wd\tboxa%
    \sbox{\tboxa}{$\m@th#1{#2}$}%
    \setlength\tdima{\the\wd\tboxa}%
    \addtolength\tdima{-\the\wd\tboxb}%
    \sbox{\tboxb}{$\m@th#1\hskip\tdima\overline{\xusebox{\tboxb}}$}%
    \rlap{\usebox\tboxb}{\usebox\tboxa}}

\newcommand*{\xusebox}[1]{\mathord{{\usebox{#1}}}}

\makeatother

\renewcommand{\bar}[1]{\,\oversymb{#1}}

\begin{document}
\newcommand{\defeq}{\vcentcolon=}
\newcommand{\eqdef}{=\vcentcolon}

\title{On the Linear Belief Compression of POMDPs\thanks{The research leading to these results was funded by the Engineering and Physical Sciences Research Council, UK (EPSRC) under project no. EP/G069840/1 and was partially supported by the EC FP7 projects Spacebook (ref. 270019) and JAMES (ref. 270435).}
\subtitle{ A re-examination of current methods}
}

\titlerunning{On the Linear Belief Compression for POMDPs}        

\author{Zhuoran~Wang \and Paul~A.~Crook\and Wenshuo~Tang \and Oliver~Lemon}

\authorrunning{Z. Wang, P. A. Crook, W. Tang, O. Lemon} 

\institute{Z. Wang \and  W. Tang \and O. Lemon \at
              Interaction Lab, MACS, Heriot-Watt University, Edinburgh, EH14 4AS, United Kingdom\\
               \email{\{zhuoran.wang; wt92; o.lemon\}@hw.ac.uk} \\
               P. A. Crook \at Microsoft, One Microsoft Way, Redmond, WA 98052, United States\\
               \email{pacrook@microsoft.com}  
}

\date{Received: date / Accepted: date}

\maketitle

\begin{abstract}

Belief compression improves the tractability of 
  large-scale partially observable Markov decision processes (POMDPs) by finding projections from high-dimensional
  belief space onto low-dimensional approximations, where solving to obtain action selection policies
  requires fewer computations. 
 This paper develops a unified theoretical framework to analyse
  three existing linear belief compression approaches, including value-directed compression and two 
  nonnegative matrix factorisation (NMF) based algorithms. The results indicate that all the three
  known belief compression methods have their own critical deficiencies. Therefore, projective NMF belief compression
 is proposed (P-NMF), aiming to overcome the drawbacks of the existing techniques.
  The performance of the proposed algorithm is examined on four POMDP problems of reasonably 
  large scale, in comparison with existing techniques. Additionally, the competitiveness of
  belief compression 
 is compared empirically to a state-of-the-art heuristic search-based POMDP solver and their
 relative merits in solving large-scale POMDPs are investigated.
  

\end{abstract}

\section{Introduction}

Making decisions in dynamic environments is one of the core problems of artificial intelligence. In many cases, it requires not only evaluation of the reward or cost of the intermediate action, but also consideration of the long term effect of a sequence of choices made in the future. If the true states of the system can be perfectly identified, Markov decision processes (MDPs) can be used to handle such planning problems efficiently. However, in real-world applications, the true system states are not always fully observable. Uncertainties may come from different sources, e.g. noisy sensors in a robotic system, or speech recognition errors in a spoken dialogue system, etc., which motivates the utilisation of probabilistic techniques to track the system states.

The partially observable Markov decision process (POMDP) has been proven to be a powerful tool for modelling sequential decision making problems under uncertainty. It generalises the standard MDP to the case where an agent cannot directly observe the underlying states but has to maintain a probability distribution  (called a belief) over all possible states based on noisy observations. The optimal policy of a POMDP then specifies an action for each possible belief to maximise expected discounted future reward. 

However, the exact solution for the policy optimisation problem of a POMDP is computationally intractable \citep{Cassandra97}. Polynomial-time approximation algorithms can be achieved using point-based value iteration (PBVI) techniques \citep{Pineau03}, which successively estimate the value function by updating the value and its gradient only at the points of a witness point set. In this case, the dimensionality of the belief space dominates the efficiency of the algorithms (see  \citep{Pineau03,Spaan05,Smith05}). 

Factored POMDPs explored in various previous studies (see below) provide a general direction for improving the tractability of large-scale problems via dimension reduction. The essential idea behind the factorisation is to decompose the original instantiations of the state, action and observation variables in a POMDP into their respective smaller sets of factor variables. Then conditional independence or context-specific independence among those factor variables can be exploited to achieve a more compact representation, using corresponding techniques such as dynamic Bayesian networks \citep{Williams05,Thomson10,Hoey10}, decision trees \citep{Boutilier96,Boutilier00} or algebraic decision diagrams \citep{Hanse00, Shani08}. Unfortunately, such factored representations do not necessarily result in efficient policy implementations. Although for particular types of POMDP problems we will also be able to express the transition, observation and reward functions in a compact form with respect to their respective factor variables and optimise the policies in lower-dimensional spaces \citep{Ong10,Sim08,Poupart05}, this method does not generalise to all domains by default. 

Belief compression provides an alternative solution to reduce the computation cost for POMDP policy optimisation by projecting the high-dimensional belief space into a low-dimensional one, using an automatically obtained projection basis. Main contributions in this area include exponential family principal component analysis (EPCA)-based compression \citep{Roy05}, value-directed compression (VDC) \citep{Poupart02,Poupart05} and  nonnegative matrix factorisation (NMF)-based compression \citep{Li07,Theocharous10}. The EPCA approach achieves a non-linear compression by exploiting the sparsity of the belief space based on sampled beliefs, whilst VDC produces a linear projection in a value-directed manner such that a belief and its compression will obtain an (approximately) identical value. Due to the nature of linear projection, VDC has the advantage that the piecewise linear and convex (PWLC) property of the value function remains after the compression, hence PBVI algorithms can be directly applied to solve the compressed POMDPs. Benefiting from the insights behind both EPCA and VDC, the NMF-based algorithms seek a projection basis (also based on sampled beliefs) that yields a low-rank approximation of the belief space, and use it to construct a linear compression to preserve convenience for PBVI. 

In this paper, after reviewing some background knowledge of POMDPs (\S 2), we develop a unified theoretical framework to analyse linear belief compression algorithms in general (\S 3). To the best of our knowledge, such an analysis has not been reported before. After this, the results are employed to examine three existing linear POMDP compression algorithms separately, including VDC \citep{Poupart05}, orthogonal NMF (O-NMF) compression \citep{Li07}, and locality preserving NMF (LP-NMF) compression \citep{Theocharous10} (which are the only three linear belief compression methods that we are aware of). Our findings show that all the three existing models have their own critical deficiencies. For VDC (\S4), not only can the compressed value function violate the contractive property of a valid Bellman recursion \citep{Bellman57}, and therefore can diverge to infinity in the worst case (even when the compression error is extremely small), but also the lack of nonnegativity constraints on its compression basis can confuse the pruning procedure in PBVI and drive the algorithm to an ill converging point. 
On the other hand, both the O-NMF and LP-NMF approaches share the common drawback that compression error does not directly relate to value loss (\S5), which results in good compressions not necessarily leading to promising policies.
Therefore, a novel projective NMF belief compression algorithm (P-NMF) is proposed (\S6), aiming to revise the deficiencies of the existing techniques.   
Experimental results on four POMDP problems of reasonably large scale show that the proposed model outperforms the existing techniques (\S7.1). In addition, we also investigate the practical effectiveness of belief compression in solving large-scale POMDPs, in comparison with a state-of-the-art (uncompressed) POMDP solver called SARSOP \citep{Kurniawati08} (\S7.2), before we conclude (\S8).


\section{POMDP Basics}

A POMDP is a tuple $\langle \mathcal{S},\mathcal{A},\mathcal{Z},T,\Omega,R,\eta\rangle$, where the components are defined as follows. $\mathcal{S}$, $\mathcal{A}$ and $\mathcal{Z}$ are the sets of states, actions and observations respectively. The transition function $T(s'|s,a)$ defines the conditional probability of transiting from state $s\in \mathcal{S}$ to state $s'\in \mathcal{S}$ after taking action $a\in \mathcal{A}$. The observation function $\Omega(z|s,a)$ gives the probability of the occurrence of observation $z\in \mathcal{Z}$ in state $s$ after taking action $a$. $R(s,a)$ is the reward function specifying the immediate reward of a state-action pair. Whilst, $0<\eta<1$ is a discount factor. In this paper, we will focus on POMDPs with discrete state, action, and observation spaces.

A standard POMDP operates as follows. At each time step, the system is in an unobservable state $s$, for which only an observation $z$ can be received. A distribution over all possible states is therefore maintained,  called a belief, denoted by  $b$, where the probability of the system being in state $s$ is $b(s)$. Based on the current belief, the system selects an action $a$, receives a reward $R(s,a)$ and transits to a new (unobservable) state $s'$ where it receives an observation $z'$. Then the belief is updated to $b'$ based on $z'$ and $a$ as follows:
\begin{eqnarray}\label{equation:belief}
b'(s')=\mathrm{Pr}(s'|z',a,b)=\frac{1}{\mathrm{Pr}(z'|a,b)}\Omega(z'|s',a)\sum_s T(s'|a,s)b(s)
\end{eqnarray}
where $\mathrm{Pr}(z'|a,b)=\sum_{s'}\Omega(z'|s',a)\sum_s T(s'|a,s)b(s)$ is a normalisation factor. 
\subsection{Policy and Value Function}
A policy $\pi$ is defined as a mapping that maps each belief $b$ to an action $a=\pi(b)$.
The value function of a given policy $\pi$ and given starting point $b_0$ is the expected sum of discounted rewards, calculated as:
\begin{eqnarray}
V^{\pi}(b_0)=E\left[\sum_{t=0}^n\eta^tr_{\pi(b_t)}(b_t)\right]
\end{eqnarray}
where $n$ is the planning horizon (possibly inifinite) and $r_{\pi(b_t)}(b_t)$ is the immediate reward obtained at time $t$ using policy $\pi$. The objective of POMDP-based planning is to determine an optimal policy $\pi^*=\arg\max_{\pi}V^{\pi}(b)$  that maximises the value function. The value function corresponding to $\pi^*$ is usually denoted by $V^*$.

\subsection{Value Iteration}
Commonly used policy optimisation algorithms include value iteration, policy iteration and linear programming. In this paper, we will focus on value iteration related techniques. Exact value iteration recursively computes the optimal value function $V^*$ as the sequence of value functions $V_n$ starting from an initial $V_0$:  
\begin{eqnarray}\label{value-function}
V_{n+1}(b)=\max_{a\in \mathcal{A}}\left[R(b,a)+\eta\sum_{z\in \mathcal{Z}}\mathrm{Pr}(b'|b,a,z)\mathrm{Pr}(z|a,b)V_{n}(b')\right]
\end{eqnarray}
where $\mathrm{Pr}(b'|b,a,z')$ is an indicator of $b$ updating to $b'$ on action $a$ and observation $z'$. If we represent $R$ in matrix form, where $R(b,a)=b^{\top} R_{\cdot,a}$, and define the mapping $T^{a,z}$ such that $T^{a,z}_{ij}=T(s_j|a,s_i)\Omega(z|a,s_j)$, Eq. (\ref{value-function}) can be re-written as:
\begin{eqnarray}\label{value-function-matrix}
V_{n+1}(b)\defeq\max_{a\in \mathcal{A}}\left[b^{\top} R_{\cdot,a}+\eta\sum_{z\in \mathcal{Z}}V_n(b^{\top} T^{a,z})\right]
\end{eqnarray}
Such a recursion is usually expressed in functional form:
\begin{eqnarray}
V_{n+1}&=&HV_n\label{bellman}
\end{eqnarray}
where $H$ is the Bellman backup operator \citep{Bellman57}. Eq. (\ref{bellman}) is also called a Bellman equation or a Bellman recursion.

The above recursion converges to $V^*$, and can be represented as a piece-wise linear convex (PWLC) function:
\begin{equation}\label{value-function-alpha}
V(b)=\max_{\alpha\in\Gamma} b^{\top}\alpha
\end{equation}
where $\Gamma$ is a set of vectors called $\alpha$-vectors, with each $\alpha$-vector associated with an action, such that the action corresponding to the $\alpha$-vector maximising the value function at the current belief $b$ is the one executed by the underlying policy $\pi$.

Value iteration can then be implemented as a dynamic programming procedure to iteratively construct the $\alpha$-vectors. Furthermore, the exact Bellman backup operator $H$ can be approximated with tractable computations by considering only a finite set of sampled belief points instead of the entire reachable belief space. This is known as point-based value iteration (PBVI). Detailed introductions to PBVI algorithms are omitted in this paper, but some commonly used techniques can be found in \citep{Pineau03,Smith05, Spaan05}. 

\section{Linear Belief Compression}

We start the discussion from an ideal case where we assume that a lossless compression is achievable.
Linear belief compression can be summarised as finding a linear function $F\in\mathbb{R}^{n\times k}$ such that:
\begin{eqnarray}\label{compression}
R=F\tilde{R} \ \ \ \ \mathrm{and} \ \ \ \ T^{a,z}F=F\tilde{T}^{a,z}\ \ \ \forall a\in \mathcal{A},z\in \mathcal{Z}
\end{eqnarray}
where $n=|\mathcal{S}|$ is the dimension of the state space in the original POMDP,  $k\ll n$ is the compressed state space size, and $\tilde{R}$ and $\tilde{T}^{a,z}$ are the compressed reward and transition matrices, and can be computed as:
\begin{eqnarray}\label{compression-solution}
\tilde{R}=F^{\dag}R \ \ \ \ \mathrm{and} \ \ \ \ \tilde{T}^{a,z} = F^{\dag}T^{a,z}F\ \ \ \forall a\in \mathcal{A},z\in \mathcal{Z}
\end{eqnarray}
where $F^{\dag}\in\mathbb{R}^{k\times n}$ is some certain form of `inverse' of $F$. Here we temporarily keep this notation general, and leave its specifications for different models explained in later sections.
Let $\tilde{b}$ be the compressed belief, and:
\begin{equation}\label{compressed-belief}
\tilde{b}^{\top}=b^{\top}F
\end{equation}
For a given policy $\pi$, the value function $\tilde{V}^{\pi}$ defined for the compressed problem can then be written as:
\begin{equation}\label{compressed-value}
\tilde{V}^{\pi}(\tilde{b})=\tilde{b}^{\top} \tilde{R}_{\cdot,\pi(\tilde{b})}+\eta\sum_z\tilde{V}^{\pi}(\tilde{b}^{\top} \tilde{T}^{\pi(\tilde{b}),z})
\end{equation}

The underlying theory behind lossless linear belief compression was initially proposed by \cite{Poupart02}. We quote their theorem and proof here for convenience of further discussion.
\begin{theorem}[\citeauthor{Poupart02}]

Let $\mathcal{B}$ denote the set of all reachable beliefs for a POMDP.
Let $F$, $\tilde{R}$ and $\tilde{T}^{a,z}$ satisfy Eq. (\ref{compression}), then $V^{\pi}(b)=\tilde{V}^{\pi}(\tilde{b})$, $\forall \pi,b\in\mathcal{B}$.

\end{theorem}
\begin{proof}

Base case: let $V^{\pi}_0(b) = b^{\top}R_{\cdot,\pi(b)}$ and $\tilde{V}^{\pi}_0(\tilde{b}) = \tilde{b}^{\top}\tilde{R}_{\cdot,\pi(\tilde{b})}$, then
\begin{equation}
V^{\pi}_0(b) = b^{\top}R_{\cdot,\pi(b)}=b^{\top}F\tilde{R}_{\cdot,\pi(\tilde{b})}=\tilde{b}^{\top}\tilde{R}_{\cdot,\pi(\tilde{b})}=\tilde{V}^{\pi}_0(\tilde{b})\nonumber
\end{equation}
Induction: let $V^{\pi}_n(b)=\tilde{V}^{\pi}_n(\tilde{b})$ with $n$ stages-to-go, then
\begin{eqnarray}
V^{\pi}_{n+1}(b)&=&b^{\top} R_{\cdot,\pi(b)}+\eta\sum_zV_n^{\pi}(b^{\top} T^{\pi(b),z}) \nonumber\\
&=&b^{\top} R_{\cdot,\pi(b)}+\eta\sum_z\tilde{V}_n^{\pi}(b^{\top} T^{\pi(b),z}F)\mbox{\texttt{\hspace{1.53cm} : \scriptsize$V^{\pi}_n(b')=\tilde{V}^{\pi}_n(\tilde{b}')$}}\nonumber\\
&=&b^{\top} F\tilde{R}_{\cdot,\pi(b)}+\eta\sum_z\tilde{V}_n^{\pi}(b^{\top} F\tilde{T}^{\pi(b),z})\mbox{\texttt{\hspace{1.25cm} : \scriptsize substituting Eq. (\ref{compression})}}\nonumber\\
&=&\tilde{b}^{\top} \tilde{R}_{\cdot,\pi(\tilde{b})}+\eta\sum_z\tilde{V}_n^{\pi}(\tilde{b}^{\top} \tilde{T}^{\pi(\tilde{b}),z})=\tilde{V}^{\pi}_{n+1}(\tilde{b})\mbox{\texttt{\hspace{0.16cm} : \scriptsize substituting Eq. (\ref{compressed-belief})}}\nonumber\ \ \qed
\end{eqnarray}

\end{proof}
Theorem 1 shows that if the conditions in Eq. (\ref{compression}) hold, all policies have identical values with respect to the compressed and uncompressed POMDPs, i.e. the compression is lossless.

\subsection{A Complementary Theory of Lossless Belief Compression}
Recall the $\alpha$-vector representation of the value function in Eq. (\ref{value-function-alpha}). For a given policy $\pi$, we can express $V^{\pi}(b)=b^{\top}V^{\pi}=b^{\top}\alpha^{\pi}$, where we use $\alpha^{\pi}$ to denote the $\alpha$-vector specified by $\pi$ in computing the value of the current $b$. A similar representation is applicable to $\tilde{V}^{\pi}$ as well.  (For instance $\tilde{V}^{\pi}(\tilde{b})=\tilde{b}^{\top}\tilde{V}^{\pi}=\tilde{b}^{\top}\tilde{\alpha}^{\pi}$.)
Then by substituting Eq. (\ref{compression-solution}) and Eq. (\ref{compressed-belief}) into Eq. (\ref{compressed-value}), the value function of the compressed POMDP can be explicitly expressed in the following form:
\begin{eqnarray}
\tilde{V}^{\pi}(\tilde{b})&=&\tilde{b}^{\top}\tilde{V}^{\pi}=b^{\top}F\tilde{V}^{\pi}\label{compressed-value-explicit1} \\
&=&b^{\top} F  \tilde{R}_{\cdot,\pi(\tilde{b})}+\eta\sum_zb^{\top} F  \tilde{T}^{\pi(\tilde{b}),z}\tilde{V}^{\pi}\nonumber\\
&=&b^{\top} F F^{\dag}R_{\cdot,\pi(b)}+\eta\sum_z b^{\top} F F^{\dag}T^{\pi(b),z}F\tilde{V}^{\pi}\label{compressed-value-explicit2}
\end{eqnarray}
Let $\bar{V}^{\pi}=F\tilde{V}^{\pi}$ and $A=FF^{\dag}$, and substitute it into Eq. (\ref{compressed-value-explicit1}) and (\ref{compressed-value-explicit2}). We obtain:
\begin{equation}\label{v_bar}
\bar{V}^{\pi}(b)=b^{\top}\bar{V}^{\pi}=b^{\top} AR_{\cdot,\pi(b)}+\eta\sum_z b^{\top} AT^{\pi(b),z}\bar{V}^{\pi}
\end{equation}
Similar to the definition of the original value function $V$, the recursive function $\bar{V}$ can also be written in functional form as:
\begin{equation}\label{Bellman-A}
\bar{V}_{n+1}=\bar{H}\bar{V}_n
\end{equation}
It means that the linear belief compression defined by Eq. (\ref{compression-solution}) and (\ref{compressed-belief}) will actually result in a modified value function with the original Bellman backup operator $H$ replaced by an approximated backup operator $\bar{H}$.
After this, we can obtain the following theorem, which is more relaxed than Theorem 1.
\begin{theorem}
Let $A$ be a low-rank square matrix that can be factored into the product of two rectangular matrices (of rank $k$) as $A=FF^{\dag}$, and $\tilde{R}$ and $\tilde{T}^{a,z}$ be defined as in Eq. (\ref{compression-solution}). If there exists such an $A$ that satisfies either (i) $R=AR$ and $T^{a,z}=AT^{a,z}$, $\forall a\in\mathcal{A},z\in\mathcal{Z}$, or (ii) $b^{\top}=b^{\top}A$, $\forall b\in\mathcal{B}$,  then $V^{\pi}(b)=\tilde{V}^{\pi}(\tilde{b})$, $\forall \pi,b\in\mathcal{B}$. 
\end{theorem}
\begin{proof}
By defining $\bar{V}^{\pi}_0(b)=b^{\top}AR_{\cdot,\pi(b)}$ and  $V^{\pi}_0(b) = b^{\top}R_{\cdot,\pi(b)}$,  and substituting either condition (i) or condition (ii), we can obtain $\bar{V}^{\pi}_0(b)=V^{\pi}_0(b)$. It is straightforward to induce that if $\bar{V}^{\pi}_n(b)=V^{\pi}_n(b)$ with $n$ stages-to-go, then $\bar{V}^{\pi}_{n+1}(b)=V^{\pi}_{n+1}(b)$, by substituting either condition (i) or condition (ii) into Eq. (\ref{v_bar}). Finally, substituting the definition of $\bar{V}^{\pi}$ into Eq. (\ref{compressed-value-explicit1}) gives $\tilde{V}^{\pi}_{n+1}(\tilde{b})=V^{\pi}_{n+1}(b)$. \qed
\end{proof}

\subsection{Convergence of Lossy Belief Compression}

The above discussions are essentially based on the ideal assumption that a lossless compression exists, which is usually not the case in practical POMDP problems. However, for many problems lossy belief compression can still be employed to reduce the computational complexity of policy training (and execution). Lossy belief compression is designed to seek a projection matrix $F$ by minimising some loss criteria,  but as a consequence errors will exist between $\tilde{V}^{\pi}_{n}(\tilde{b})$ and $V^{\pi}_{n}(b)$. Moreover, such errors may propagate during value iteration, and result in a significant loss in the quality of obtained policy. The error propagation problem has been studied in depth for MDPs under reinforcement learning scenarios in previous literature \citep{Farahmand10,Antos08,Munos07}. However,  their results do not directly transfer to POMDP problems due to more complex backup procedures. Hence, in this paper we only study a basic problem: sufficient conditions of the value function Eq.\ (\ref{compressed-value}) under lossy compression being a valid Bellman equation that converges monotonically to a fixed point. Such convergence implies a bounded loss between the original and the compressed value functions.

Since $\tilde{V}^{\pi}(\tilde{b})$ and $\bar{V}^{\pi}(b)$ return the same value if $\tilde{b}$ is the compression of $b$, it is much easier to investigate the latter, which only involves adding an extra linear operator $A$ to the original value function. 
\begin{lemma}
Let $\bar{b}=(b^{\top}A)^{\top}$ and $\hat{b}=\frac{\bar{b}}{\|\bar{b}\|_1}$ (i.e. normalised $\bar{b}$), For $V$ and $\bar{V}$ defined in Eq. (\ref{bellman}) and (\ref{Bellman-A}) respectively, $\bar{H}\bar{V}(b)=\|b^{\top}A\|_1H\bar{V}(\hat{b})$.
\end{lemma}
\begin{proof}
\begin{eqnarray}
\bar{H}\bar{V}(b)&=&\max_{a,\bar{\alpha}}\left(b^{\top}AR_{\cdot,a}+\eta b^{\top}A\sum_z T^{a,z}\bar{\alpha}\right)\nonumber\\
&=&\|b^{\top}A\|_1 \max_{a,\bar{\alpha}} \left(\hat{b}^{\top}R_{\cdot,a}+\eta \hat{b}^{\top}\sum_z T^{a,z}\bar{\alpha}\right)\nonumber\\
&=&\|b^{\top}A\|_1H\bar{V}(\hat{b})\nonumber
\end{eqnarray}
\qed
\end{proof}
After this, Lemma 1 can be used to prove the following lemma.
\begin{lemma}
If $\eta\|A\|_{\infty}<1$, $\bar{V}$ defined in Eq. (\ref{v_bar}) and (\ref{Bellman-A}) is contractive, i.e. for two given value functions $U_1$ and $U_2$ and the recursion $\bar{H}$ it holds that
\[
\|\bar{H}U_1-\bar{H}U_2\|_{\infty}\leq\beta\|U_1-U_2\|_{\infty}
\]
with $0<\beta<1$ and $\|\cdot\|_{\infty}$ the supreme norm. 
\end{lemma}
\begin{proof}
\begin{eqnarray}
\|\bar{H}U_1-\bar{H}U_2\|_{\infty}&=&\|\bar{H}U_1(b)-\bar{H}U_2(b)\|_{\infty}\nonumber\\
&=&\|b^{\top}A\|_1\|HU_1(\hat{b})-HU_2(\hat{b})\|_{\infty}\nonumber\\
&\leq&\|A\|_{\infty}\|HU_1-HU_2\|_{\infty}\nonumber\\
&\leq&\eta\|A\|_{\infty}\|U_1-U_2\|_{\infty}\nonumber
\end{eqnarray}
where we utilise the matrix norm property $\|bA\|_1\leq\|b\|_1\|A\|_{\infty}$, and the facts that $\|b\|_1=1$ and for a standard Bellman recursion, $\|HU_1-HU_2\|_{\infty}\leq\eta\|U_1-U_2\|_{\infty}$.\qed
\end{proof}
The contraction property ensures that the vector space defined by the compressed value function is complete. Therefore, the space of such value functions together with the supreme norm form a Banach space, and the Banach fixed-point theorem ensures that a single fixed point exists, to which the value recursion always converges \citep{Puterman05}. 
\begin{lemma}
$\bar{V}$ defined in Eq. (\ref{v_bar}) and (\ref{Bellman-A}) is isotonic, i.e. for two given value functions $U_1$ and $U_2$ and the recursion $\bar{H}$ it holds that
\[
U_1\leq U_2 \Rightarrow \bar{H}U_1 \leq \bar{H}U_2
\]
\end{lemma}
\begin{proof}
Let $\bar{H}U_1(b)=\bar{H}^{a_1}U_1(b)$ and $\bar{H}U_2(b)=\bar{H}^{a_2}U_2(b)$, with $a_1$ and $a_2$ denoting the actions maximising $\bar{H}U_1$ and $\bar{H}U_2$ at point $b$, respectively. 
Using this definition, we have $\bar{H}^{a_1}U_2(b)\leq \bar{H}^{a_2}U_2(b)$.

Let $\bar{b}^{a,z}=(b^{\top}AT^{a,z})^{\top}$, and $\bar{\alpha}_1^{b,a,z}$ and $\bar{\alpha}_2^{b,a,z}$ be the $\alpha$-vectors maximising the value function $U_1$ and $U_2$ at $\bar{b}^{a,z}$ respectively. The following holds
\begin{eqnarray}
U_1\leq U_2& \Rightarrow& U_1(\bar{b}^{a_1,z})\leq U_2(\bar{b}^{a_1,z}), \ \ \forall b,z \nonumber\\
&\Rightarrow&  b^{\top} AT^{a_1,z}\bar{\alpha}_1^{b,a_1,z} \leq b^{\top} AT^{a_1,z}\bar{\alpha}_2^{b,a_1,z}, \ \forall b,z\nonumber\\
& \Rightarrow& b^{\top} AR_{\cdot,a_1}+\eta\sum_z b^{\top} AT^{a_1,z}\bar{\alpha}_1^{b,a_1,z}\nonumber\\
& & \ \ \ \ \ \leq\ b^{\top} AR_{\cdot,a_1}+\eta\sum_z b^{\top} AT^{a_1,z}\bar{\alpha}_2^{b,a_1,z},\ \forall b\nonumber\\
& \Rightarrow& \bar{H}^{a_1}U_1(b)\leq \bar{H}^{a_1}U_2(b) \leq \bar{H}^{a_2}U_2(b), \ \forall b\nonumber\\
& \Rightarrow& \bar{H}U_1(b)\leq \bar{H}U_2(b),\ \ \forall b \nonumber\\
&\Rightarrow& \bar{H}U_1 \leq \bar{H}U_2\nonumber
\end{eqnarray}
\qed
\end{proof}
The isotonic property of the value function guarantees that value iteration converges monotonically.

Nevertheless, the above lemmas only consider exact value backup, which is intractable in practice. If approximate backup are taken into account, pruning is an inevitable procedure to ensure an efficient size of the $\alpha$-vector set \citep{Zhang05}.  Moreover, working in the compressed belief space, the backup operation is based on $\tilde{V}$ instead of $\bar{V}$. Since an essential step in pruning is to check the domination of an $\alpha$-vector by others, the following lemma gives a further constraint for a compressed POMDP to be efficiently solvable.

\begin{lemma}
Let $\tilde{b}$ be a compressed belief with respect to a compression function $F$, as defined in Eq. (\ref{compressed-belief}). For two arbitrary $\alpha$-vectors $\tilde{\alpha}_1$ and $\tilde{\alpha}_2$ (corresponding to the compressed value function), $\tilde{\alpha}_1\geq\tilde{\alpha}_2 \Rightarrow \tilde{b}^{\top}\tilde{\alpha}_1\geq \tilde{b}^{\top}\tilde{\alpha}_2$, $\forall \tilde{b}$, holds for an arbitrary valid POMDP, iff $F\geq0$. 
\end{lemma}
The proof is straightforward, hence is omitted here.

\section{Value-Directed Compression}

Lossless VDC is designed to seek a linear compression function $F$ that satisfies Eq.\ (\ref{compression}), which is a more restricted condition compared  to our Theorem 2.  \cite{Poupart05} proposed that such an $F$ can be obtained by iteratively exploring the Krylov subspace $\mathrm{Kr}(\{T^{a,z}\}_{a\in\mathcal{A},z\in\mathcal{Z}},R)$ to find a Krylov basis (with $k$ linearly independent column vectors). We can summarise this process as (i) initialising $F$ to the linearly independent columns of $R$, and (ii) in each iteration, multiplying every column ($F_i$) of $F$ with every $T^{a,z}$ (as $T^{a,z}F_i$), and appending the obtained vector to $F$ if it is linearly independent of the columns of the current $F$. The process ends when no more columns can be added.
\begin{table}[t]
\centering
\begin{tabular}{rl}
\multicolumn{2}{l}{{\bf Algorithm 1:} Lossless/Lossy Value-Directed Compression \citep[Chapter~4]{Poupart05}}\\
\hline\\
1: &input: $R$,\ \  $\{T^{a,z}\}_{a\in\mathcal{A},z\in\mathcal{Z}}$,\ \  $k$\ \  {\small/*truncation parameter, lossy VDC only*/} \\
2:  & $F\leftarrow\emptyset$\\
3:  & $\mathcal{C}\leftarrow R$\ \  {\small/*candidate column set*/} \\
4:  &repeat \\
5a: &\hspace{1cm}$c\leftarrow\mathcal{C}[\mathrm{first}]$\ \hspace{3.1cm} {\small/*lossless VDC*/}\\
5b: &\hspace{1cm}$c\leftarrow \arg\max_{y\in \mathcal{C}}\min_{x}\|y-Fx\|$\ \ \hspace{0.35cm}  {\small/*lossy VDC*/}\\
6: &\hspace{1cm}$F\leftarrow[F,c]$\\
7: &\hspace{1cm}for each $(a,z)\in\mathcal{A}\times\mathcal{Z}$\\
8: &\hspace{2cm}$\mathcal{C} \leftarrow[\mathcal{C},T^{a,z}c]$\\
9: &\hspace{1cm}remove columns linearly dependent to $F$ from $\mathcal{C}$\\
10: &until $\mathcal{C}=\emptyset$ or $\mathrm{length}(F)=k$\ \ {\small/*truncation, lossy VDC only*/}\\
11:&solve Eq. (\ref{compression})\ \ {\small/*lossless VDC*/} or Eq. (\ref{VDC-lossy})\ \ {\small/*lossy VDC*/}\\
&\hspace{0.75cm}to obtain $\tilde{R}$ and $\{\tilde{T}^{a,z}\}_{a\in\mathcal{A},z\in\mathcal{Z}}$\\
12: &return $F$, $\tilde{R}$, $\{\tilde{T}^{a,z}\}_{a\in\mathcal{A},z\in\mathcal{Z}}$\\
\end{tabular}
\end{table}

Clearly, a lossless compression is achievable only if the Krylov subspace is low-rank. In a more general case, it was suggested that one can greedily select $k$ basis vectors in a forward-search manner to approximately minimise the residual errors of their predictions on the remaining vectors in the Krylov subspace, which is known as lossy VDC \citep{Poupart05}.  After obtaining $F$, $\tilde{R}$ and $\tilde{T}^{a,z}$ can be computed by either solving Eq. (\ref{compression}) for lossless VDC, or the following regression problem for lossy VDC.
\begin{equation}\label{VDC-lossy}
\tilde{R}=\arg\min_{\hat{R}}\|R-F\hat{R}\|_{\mathrm{F}}\mbox{\ \ and\ \;}\tilde{T}^{a,z}=\arg\min_{\hat{T}^{a,z}}\|T^{a,z}F-F\hat{T}^{a,z}\|_{\mathrm{F}} \mbox{\ \ $\forall a\in\mathcal{A}$, $z\in\mathcal{Z}$}
\end{equation}
where $\|\cdot\|_{\mathrm{F}}$ denotes the Frobenius norm. 

For the convenience of future discussions, we list the pseudo-code for loss- less and lossy VDC in Algorithm 1. Note here, we adapt the representation of the lossless VDC algorithm to make it more comparable to lossy VDC, but it works exactly in the same way as the original algorithm presented in \citep{Poupart05}.
In standard numerical computation libraries (e.g.\ LAPACK \citep{Anderson99}), an overdetermined linear equation is usually solved as a least-squares problem, which suggests that Eq.\ (\ref{compression}) and Eq.\ (\ref{VDC-lossy}) are treated the same in practice, and both lead to a solution in the form of Eq.\ (\ref{compression-solution}), with $F^{\dag}$ being the pseudo-inverse of $F$ in this case. Therefore, comparing the two VDC algorithms, one can see that the essential difference between them is their column selection strategy.

\subsection{Deficiency of VDC}\label{sec:dificiency-vdc}

Firstly, VDC does not guarantee a nonnegative $F$ unless $R$ is nonnegative. Therefore, according to Lemma 4, if PBVI is applied to the compressed problem, some $\alpha$-vectors may be mistakenly pruned out, resulting in the algorithm converging to a non-optimal policy. Secondly, as $F$ is unregularised in VDC, $\|FF^{\dag}\|_{\infty}$ can be arbitrarily large. In the worst case, the compressed value function may diverge to infinity. These arguments seem to contradict the proofs in \citep{Poupart05}, so we give more detailed explanations as follows. 

Theorem 1 is derived by assuming that the compression is lossless. However, due to numerical errors in practical computations, a totally `error-free' solution never really exists. Especially for lossless VDC, because of the lack of consideration of system conditioning, it tends to be less robust to numerical errors than lossy VDC.\footnote{Detailed analysis on the numerical stabilities of VDC is out of the main scope of this paper, but can be found in the supplementary material.} In other words, residual errors always occur in the compressed value function. Furthermore, although \cite{Poupart05} also developed an error bound for lossy VDC, that:
\begin{equation}\label{VBound}
\|V^{*}-F\tilde{V}^{*}\|_{\infty}\leq\frac{1}{1-\eta}\epsilon_R+\frac{\eta\epsilon_T}{1-\eta}|\mathcal{Z}|\|\tilde{V}^*\|_{\infty}
\end{equation}
where $\epsilon_R=\|R-F\tilde{R}\|_{\infty}$, $\epsilon_T=\max_{a,z}\|T^{a,z}F-F\tilde{T}^{a,z}\|_{\infty}$, in the case where $\tilde{V}^*$ is a diverging function, the bound itself is infinitely large.

\subsection{Empirical Evidence}

To support our argument on the deficiencies of VDC, we investigated its performance on two benchmark problems, Coffee \citep{Boutilier96} and Hallway2 \citep{Littman95}.
\begin{figure}[p]
\includegraphics[width=0.95\linewidth]{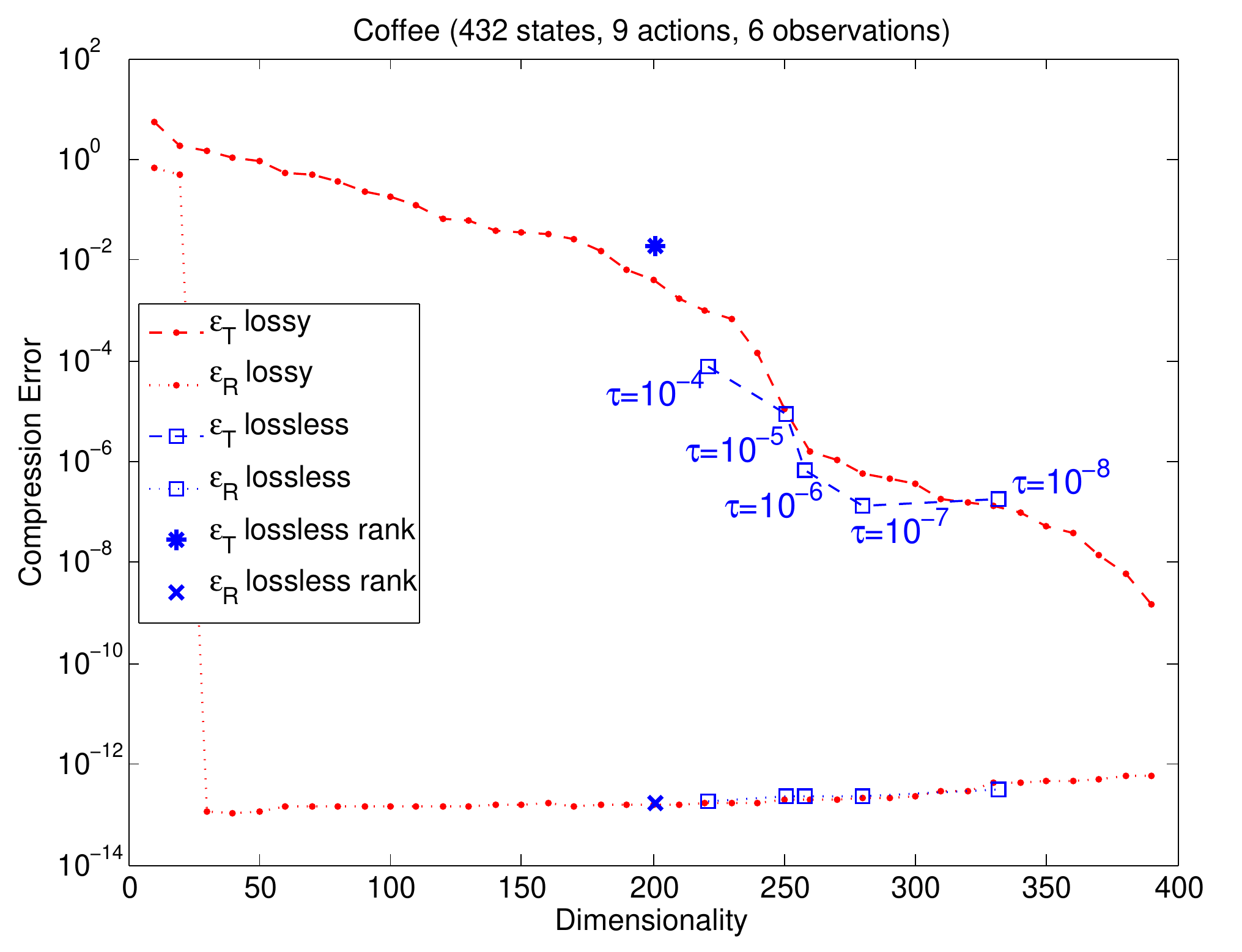}
\includegraphics[width=0.95\linewidth]{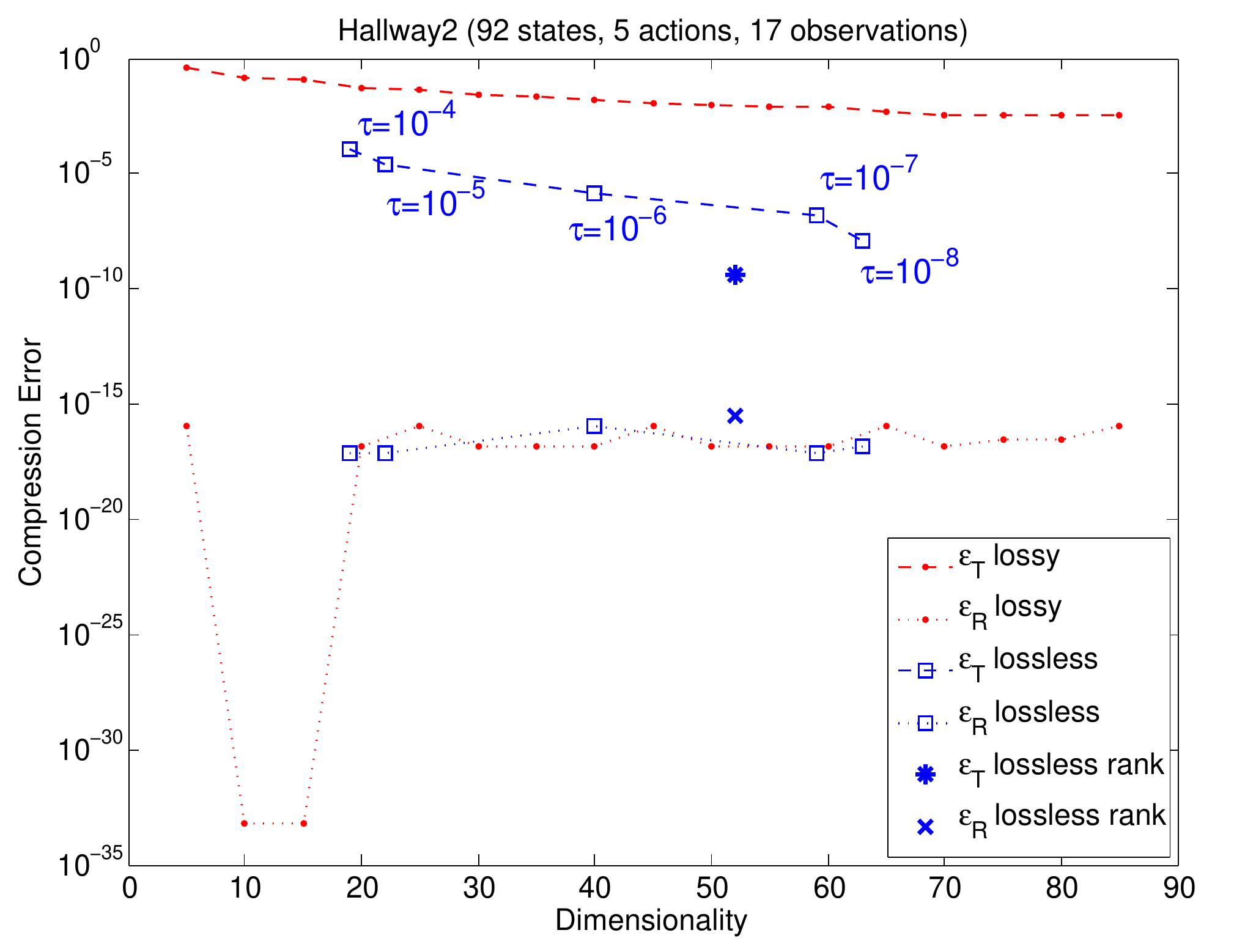}
\caption{Compression errors of lossless and lossy VDC on Coffee and Hallway2.}\label{fig:coffee-hallway2-error}
\end{figure}
\begin{figure}[p]
\includegraphics[width=0.95\linewidth]{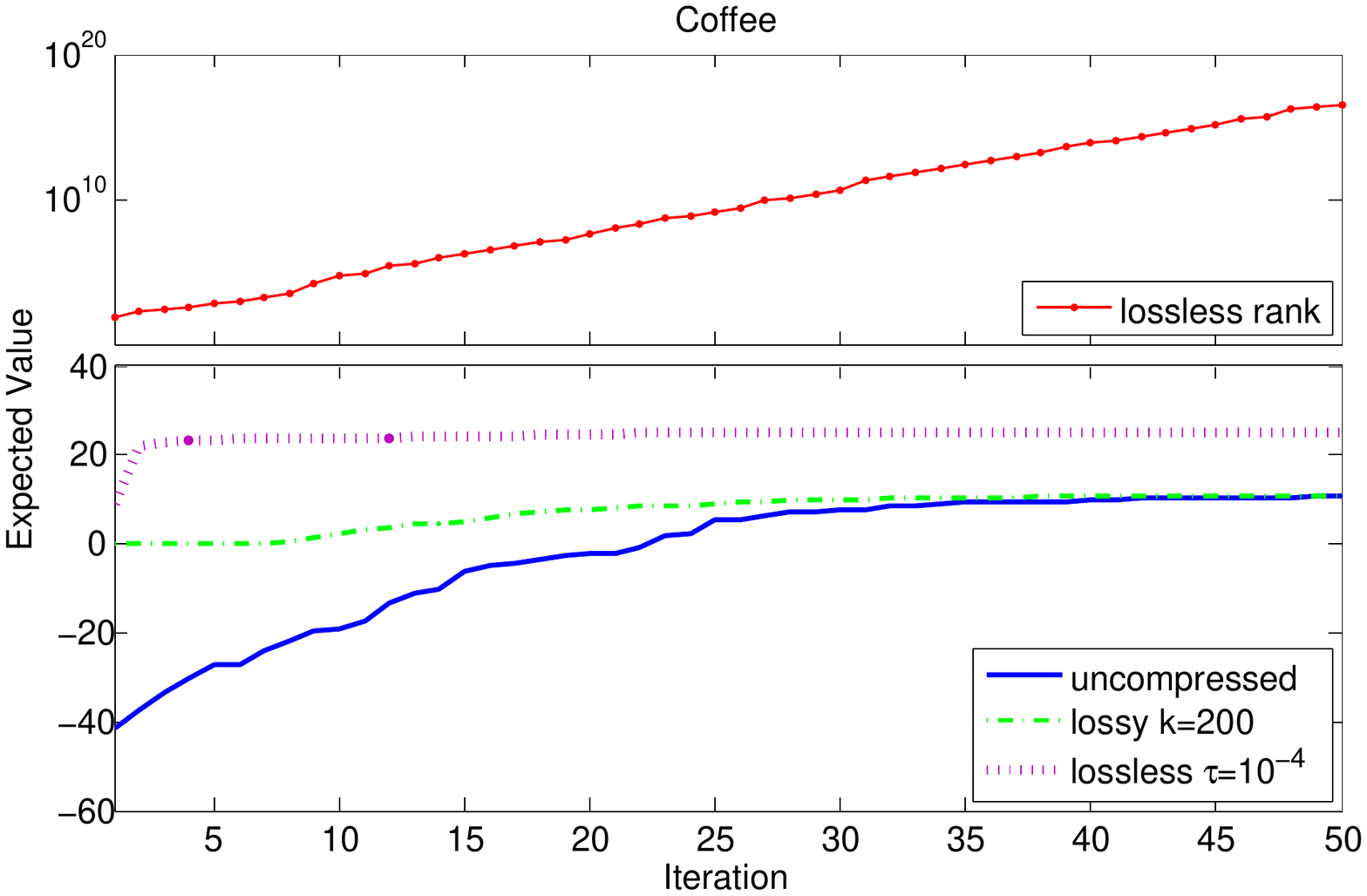}
\includegraphics[width=0.95\linewidth]{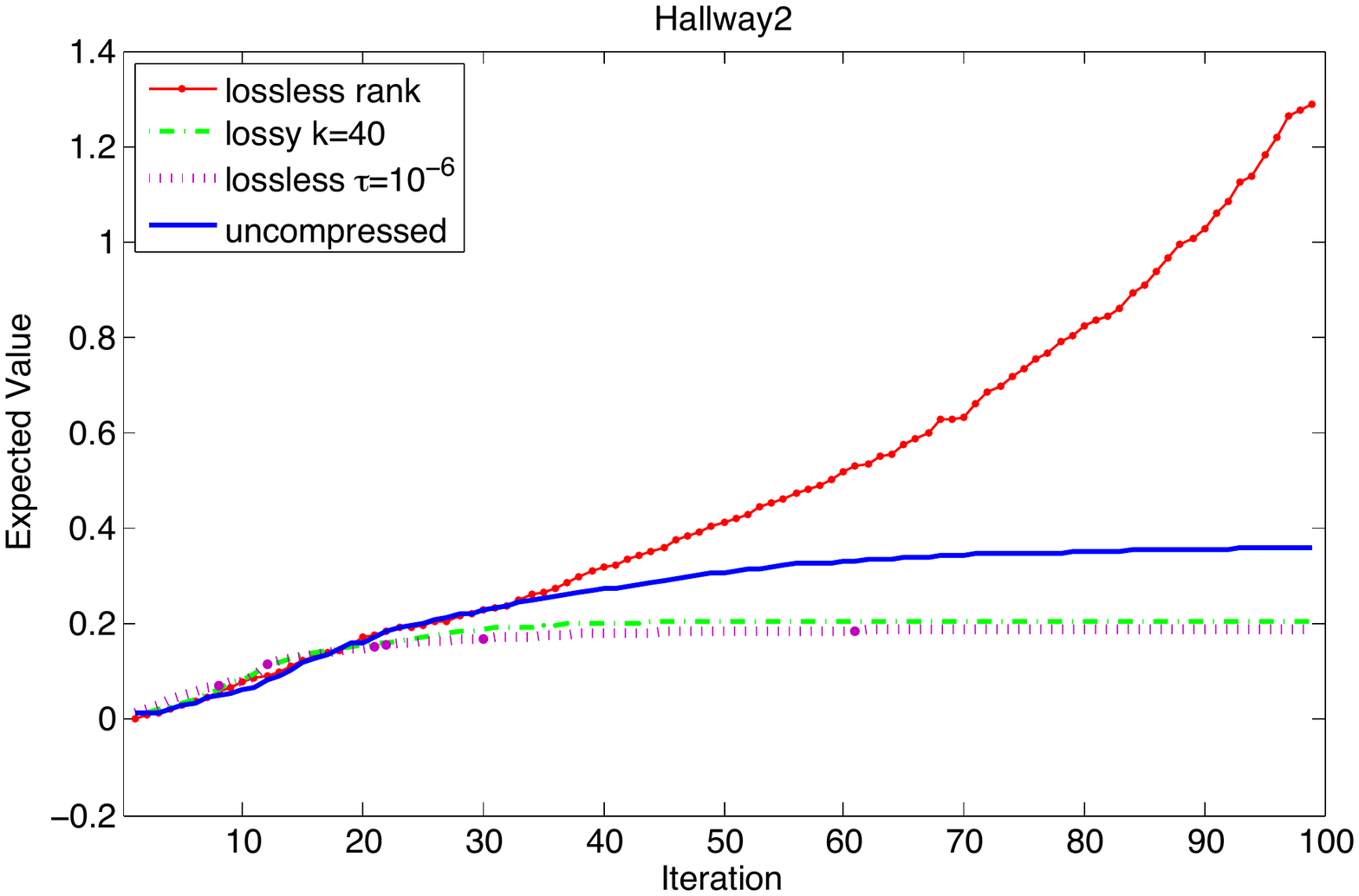}
\caption{Expected value during value iterations on Coffee and Hallway2.}\label{fig:coffee-hallway2-value}
\end{figure}
\begin{table}[p]
\centering
\begin{tabular}{|c|c|c|c|c|}
\cline{2-5}
\multicolumn{1}{c|}{Coffee} & Lossless/rank & Lossless/$\tau=10^{-4}$ & Lossy/$k=200$ & Uncompressed\\
\hline
Reward &$-5.52\pm2.11$ &$-14.36\pm 2.34$&$10.88\pm0.17$ &$11.16\pm0.14$\\
\hline
\multicolumn{5}{c}{\vspace{0.05cm}}\\
\cline{2-5}
\multicolumn{1}{c|}{Hallway2} & Lossless/rank & Lossless/$\tau=10^{-6}$ & Lossy/$k=40$ & Uncompressed\\
\hline
Reward &$0.01\pm0.00$&$0.23\pm0.02$ &$0.24\pm0.04$&$0.34\pm0.02$ \\
\hline
\end{tabular}
\caption{Average sampled rewards on Coffee and Hallway2.}\label{tab:reward-vdc}
\end{table}

In the implementation of VDC, we experiment with two ways of judging the linear dependence between a new obtained vector $c$ and the existing columns of $F$. The first method is to set a threshold $\tau$ for the least-squares residual $r=\|c-Fw\|_2$, where $w=\arg\max_{\hat{w}}\|c-F\hat{w}\|_2$. Concretely, $c$ will be appended to $F$ only if $r\geq\tau$. The second way of doing this is to check whether $\mathrm{rank}([F,c])>\mathrm{rank}(F)$ with $\mathrm{rank}(\cdot)$ denoting the numerical rank of a matrix. The quantities $\epsilon_R$ and $\epsilon_T$ as defined in Eq.\ (\ref{VBound}) are taken as measures of compression error. The compression quality of VDC with respect to different residual thresholds $\tau$ (for lossless VDC) and truncation levels $k$ (for lossy VDC) on the two benchmark problems are illustrated in Figure \ref{fig:coffee-hallway2-error}. Note that in the Coffee problem, the Krylov iteration for the rank-based lossless VDC finishes when 201 columns in $F$ are obtained, however, there is still a significant residual error $\epsilon_T$ at this point. This is an example of the numerical instability issue of VDC.

Some example points in Figure \ref{fig:coffee-hallway2-error} are selected to examine the policy quality obtained from the corresponding compressed POMDPs. To roughly ensure the same compression level for lossless and lossy VDC, we choose lossless VDC with $\tau=10^{-4}$ (221 dimensions) and lossy VDC with $k=200$ for Coffee, and lossless VDC with $\tau=10^{-6}$ (40 dimensions), and lossy VDC with $k=40$ for Hallway2. In addition, we evaluate the rank-based VDC (201 dimensions for Coffee and 52 dimensions for Hallway2) and the original POMDPs for both problems as well.  Perseus \citep{Spaan05} is employed to solve the compressed and uncompressed POMDPs here. The the expected value growth for each algorithm during the value iteration procedures are shown in Figure \ref{fig:coffee-hallway2-value}, where all the experiments are based on 5000 sampled belief points. We also sample 1000 decision trajectories for each learned policy to compute an average reward, which gives an insight into the actual quality of the obtained policy, since the expected values for compressed problems can be unreliable according to our discussion in Section \ref{sec:dificiency-vdc}. The above policy learning and evaluation procedure is repeated five times for each task, and Table \ref{tab:reward-vdc} shows the means and standard deviations of the average rewards.

These two problems demonstrate all the deficiencies of VDC mentioned above. Firstly, as shown in Figure \ref{fig:coffee-hallway2-value}, the rank-based VDC results in a diverging value function in both tasks. Especially in Hallway2, even when  the compression error is in the $10^{-10}$ level, it may still cause a significant (possibly unbounded) loss in the value function. Secondly, as can be seen in the Coffee problem, the policy learning for lossless VDC with $\tau=10^{-4}$ converges to an unreasonably high value, but the corresponding average sampled reward (in Table \ref{tab:reward-vdc}) is extremely low. This is due to the pruning procedure being confused by the negative elements in $F$. More concretely speaking, some $\alpha$-vectors that should be dominated by all the others are mistaken as the dominating ones.

The above theoretical and empirical analysis suggests that the performance of VDC is not guaranteed in practical usage, hence we stop further experiments with it in this paper.

\section{Orthogonal NMF for Belief Compression}

Orthogonal NMF (O-NMF) based belief compression, introduced by \cite{Li07},  explores an alternative direction by taking advantage of the possible low-rankness of the reachable belief space $\mathcal{B}$ defined by a POMDP.  It seeks a nonnegative factorisation of a sampled set of beliefs $B=[b_1,b_2,\ldots,b_m]$ subject to an orthogonal constraint, such as:
\begin{equation}\label{onmf}
B\approx F\tilde{B} \ \ \mathrm{s.t.}\ \  FF^{\top}=I, \ F\geq0\ \mathrm{ and }\ \tilde{B}\geq0 
\end{equation}
where $\tilde{B}$ denotes the set of the compressed beliefs and $I$ is the identity matrix. The compressed reward and transition matrices are then constructed by substituting $F^{\dag}=F^{\top}$ into Eq. (\ref{compression-solution}).

\subsection{Deficiency of O-NMF Belief Compression}\label{onmf-deficiency}
An obvious deficiency of the above formulation is that the orthogonal constraint can never be satisfied in practice, since the compression matrix $F$ is of low-rank. Therefore, the algorithm proposed in  \citep{Li07} actually only solves the following optimisation problem.
\begin{equation}\label{onmf2}
\min_{F\geq0,\tilde{B}\geq0} \|B-F\tilde{B}\|^2_{\mathrm{F}}+\lambda\|I-FF^{\top}\|^2_{\mathrm{F}}
\end{equation}
where $\lambda\geq0$ is a coefficient balancing the weights of the two loss functions in the objective.\footnote{Although \cite{Li07} suggested that in particular way of selecting $\lambda$, Eq.\ (\ref{onmf2}) is equivalent to Eq.\ (\ref{onmf}), it is well-known that $FF^{\top}=I$ is impossible for any matrix $F\in\mathbb{R}^{n\times k}$ with $k<n$, as $\mathrm{rank}(FF^{\top})=\mathrm{rank}(F)\leq k<n=\mathrm{rank}(I)$.}

Nevertheless, when compared to VDC, O-NMF belief compression has the superiority that the nonnegative $F$ ensures that PBVI works properly (Lemma 4). Moreover, minimising the Frobenius norm difference between $FF^{\top}$ and $I$ approximately controls the scale of $\|FF^{\top}\|_{\infty}$ due to the equivalence of norms in finite dimension, which preserves the convergence of the compressed value function  (Lemma 2). The fact that essentially makes O-NMF belief compression work can be understood as follows. Assume we have a compressed belief $\tilde{b}$ such that the original belief $b=F\tilde{b}$. Then the compressed value function for a given policy $\pi$ can be computed as:
\begin{eqnarray}
\tilde{V}^{\pi}(\tilde{b})&=&\tilde{b}^{\top}\tilde{R}_{\cdot,\pi(\tilde{b})}+\eta\sum_z\tilde{V}^{\pi}(\tilde{b}^{\top}\tilde{T}^{\pi(\tilde{b}),z})\nonumber\\
&=&b^{\top}R_{\cdot,\pi(b)}+\eta\sum_z\tilde{V}^{\pi}(b^{\top}T^{\pi(b),z}F)\nonumber\\
&=&b^{\top}R_{\cdot,\pi(b)}+\eta\sum_z\bar{V}^{\pi}(b^{\top}T^{\pi(b),z})
\end{eqnarray}
Note here, although the compressed belief $\tilde{b}$ here is not obtained as in Eq.\ (\ref{compressed-belief}), the corresponding $\tilde{V}$ can still be related to $\bar{V}$.
Letting $b^{\pi(b),z}=b^{\top}T^{\pi(b),z}$, we have:
\begin{equation}\label{nmf-loss}
V^{\pi}(b)-\tilde{V}^{\pi}(\tilde{b})=\eta\sum_z [V^{\pi}(b^{\pi(b),z})-\bar{V}^{\pi}(b^{\pi(b),z})]
\end{equation}
Going further, we can prove the following bound on the difference between $\bar{V}$ and $V$. 
\begin{theorem}
For value functions $V$ and $\bar{V}$ defined in Eq.\ (\ref{value-function-matrix}) and Eq.\ (\ref{v_bar}) respectively, if $\eta\|A\|_{\infty}<1$, then the following bound holds.
\[
	\|V-\bar{V}\|_{\infty}\leq\frac{\|I-A\|_{\infty}}{1- \eta\|A\|_{\infty}}\left(\|R\|_{\infty}+\eta|\mathcal{Z}|\|V^*\|_{\infty}\right)
\]
\end{theorem}
\begin{proof}
\begin{eqnarray}
\|V-\bar{V}\|_{\infty}&=&\max_{\pi}\|V^{\pi}-\bar{V}^{\pi}\|_{\infty} \nonumber\\
&\leq&\max_{\alpha,\bar{\alpha},a}\|R_{\cdot,a}+\eta\sum_z T^{a,z}\alpha-AR_{\cdot,a}-\eta\sum_z AT^{a,z}\bar{\alpha}\|_{\infty}\nonumber\\
&=&\max_{\alpha,\bar{\alpha},a}\|(I-A)R_{\cdot,a} + \eta\sum_z T^{a,z}\alpha-\eta\sum_z AT^{a,z}\alpha \nonumber\\
& &\mbox{\hspace{2cm}} + \eta\sum_z AT^{a,z}\alpha-\eta\sum_z AT^{a,z}\bar{\alpha}\|_{\infty}\nonumber\\
&\leq&\|(I-A)R\|_{\infty}+\max_{\alpha,a}\|\eta (I-A)\sum_zT^{a,z}\alpha\|_{\infty}\nonumber\\
& &\mbox{\hspace{2cm}} + \max_{\alpha,\bar{\alpha},a}\|\eta A\sum_z T^{a,z}(\alpha-\bar{\alpha})\|_{\infty}\nonumber\\
&\leq& \|I-A\|_{\infty}\|R\|_{\infty}+\eta|\mathcal{Z}| \|I-A\|_{\infty}\|V^*\|_{\infty}+\eta\|A\|_{\infty}\|V-\bar{V}\|_{\infty}\nonumber\\
&\leq&\frac{\|I-A\|_{\infty}}{1- \eta\|A\|_{\infty}}\left(\|R\|_{\infty}+\eta|\mathcal{Z}|\|V^*\|_{\infty}\right)\nonumber
\end{eqnarray}
where we apply Lemma 2 and the facts that $\|T^{a,z}\|_{\infty} \leq1$ and $\|V\|_{\infty}\leq\|V^*\|_{\infty}$.
\qed
\end{proof}

Theorem 3 implies that the O-NMF method (with $A=FF^{\top}$) minimises the upper bound of the value loss caused by compression (as $\|I-A\|_{\infty}\leq\sqrt{n}\|I-A\|_{\mathrm{F}}$). However, such a bound can be quite loose in practice. The drawback of O-NMF belief compression is that it does not directly relate the error in compressing a belief $b$ to the loss in its corresponding compressed value function, because Eq.\ (\ref{nmf-loss}) indicates that even if $b=F\tilde{b}$ holds, $V^{\pi}(b)-\tilde{V}^{\pi}(\tilde{b})\neq0$ since $V^{\pi}(b')= \bar{V}^{\pi}(b')$ does not hold in general for 
the successor beliefs $b'$ transited to from $b$.

\subsection{On the Locality Preserving NMF Belief Compression}

\cite{Theocharous10} proposed an extension to the O-NMF method, which is formulated as follows. Firstly, the sampled belief set $B$ is sub-sampled into a smaller set $B'$ by only including those points in the original $B$ that are at least $\delta$ apart in terms of Euclidean distance, where $\delta$ is a pre-fixed threshold.  Then, in finding a factorisation of $B'$, the original Frobenius norm loss is replaced by an unnormalised KL-divergence loss between $B'$ and $F\tilde{B}'$. After this, an extra risk is introduced, which measures the distance (with respect to symmetric KL-divergence) between each pair of the compressed beliefs weighted by a neighbourhood graph among the original belief points.  The neighbourhood graph is constructed by connecting every belief point with its $K$-nearest neighbourhood (KNN). Finally, similar to O-NMF, $F^{\dag}$ can be obtained by approximating $I\approx FF^{\dag}$. The insight behind this algorithm is that the neighbourhood graph inspired risk function forces two compressed beliefs to be close to each other if their corresponding original beliefs are so. Hence, it is named ``locality preserving'' NMF (LP-NMF) belief compression.

According to our discussion in Section \ref{onmf-deficiency}, LP-NMF shares both the advantages and the drawbacks of the O-NMF method. In addition, the locality preserving property is only intuitively motivated, as the initial closeness among the belief points will be lost soon during value iteration due to recursive influence of $T^{a,z}$. Furthermore, KL-divergence does not tend to be a good measure for selecting the linear compression matrix $F$. As shown in \citep{Poupart00}, the quality of the policies resulting from approximate belief monitoring can be significantly lower than the original policy even when the KL-divergence remains fairly small, whilst policy quality can be unaffected when KL-divergence is large.

\section{Projective NMF for Belief Compression}
To address the drawbacks of O-NMF belief compression, we propose a novel projective NMF belief compression algorithm motivated by Theorem 2 and Lemma 2. 

Again, we start from an ideal case. Assume that for a POMDP, the reachable belief space $\mathcal{B}$ is row-rank (say $\mathrm{rank}(\mathcal{B})=k\ll n$). Let $B$ be a sampled set of beliefs, and $\mathcal{B}\subset\mathrm{span}(B)$ (i.e.\ $\mathrm{rank}(B)=k$). If we can find a matrix $F\in\mathbb{R}^{n\times k}$ and $F\geq0$, such that:
\begin{equation}
B=FF^{\top}B \ \ \mbox{and} \ \ \|FF^{\top}\|_{\infty}<\frac{1}{\eta} 
\end{equation}
letting $F^{\dag}=F^{\top}$ and substituting it into Eq.\ (\ref{compression-solution}) gives a lossless compression of the original POMDP. The correctness of this proposition is easy to check. As $B$ contains a linear basis of $\mathcal{B}$, there exists a weight vector $w$ such that $b=Bw$, $\forall b\in\mathcal{B}$. Therefore, $FF^{\top}b=FF^{\top}Bw=Bw=b$, $\forall b\in\mathcal{B}$ holds. 
Moving to the imperfect case, if for a subset of the sampled beliefs $\bar{B}\subset B$, $\bar{B}=FF^{\top}\bar{B}$ is achievable, then $\forall b\in\mathrm{span}(\bar{B})$, $b=FF^{\top}b$. 
Hence, the insight here is that we attempt to reduce the compression error directly from the value function, instead of minimising a loose error bound as O-NMF belief compression does.

Intuitively, the optimisation problem to seek the $F$ for the proposed method could be formulated as:
\begin{equation}
\min_{F\geq 0}\ \|B-FF^{\top}B\|_1\ \mathrm{\ \ s.t.\ \ }\ \|FF^{\top}\|_{\infty}\leq\frac{1}{\eta}
\end{equation}
However, we suggest using the Frobenius norm instead of the $L_1$/$L_{\infty}$ matrix norms for efficiency purposes,  as minimising the latter requires repeatedly solving linear programming problems, which is computationally much more expensive. Hence, the problem becomes:
\begin{equation}\label{p-nmf}
\min_{F\geq 0}\ \frac{1}{2}\|B-FF^{\top}B\|^2_{\mathrm{F}}+\frac{\lambda}{2}\|FF^{\top}\|^2_{\mathrm{F}}
\end{equation}
where $\lambda$ is a regularisation coefficient to approximately control the scale of $FF^{\top}$. 
Our preliminary experiments indicate that $\lambda$ can be empirically selected and that $\eta\|FF^{\top}\|_{\infty}$ being slightly greater than 1 usually does not affect the convergence of value iteration in practice. 

Eq.\ (\ref{p-nmf}) forms a regularised projective NMF problem. As suggested by \cite{Yang10}, this class of problems are convex and thus can be solved via gradient descent as follows.
Define the objective function:
\begin{equation}
g(F) = \frac{1}{2}\|B-FF^{\top}B\|^2_{\mathrm{F}}+\frac{\lambda}{2}\|FF^{\top}\|^2_{\mathrm{F}}
\end{equation}
The the constrained gradient of $g$ for $F$ is given by:
\begin{equation}
\frac{\partial g}{\partial F_{i,j}}= -2(BB^{\top}F)_{i,j}+(FF^{\top}BB^{\top}F)_{i,j}+(BB^{\top}FF^{\top}F)_{i,j}+2\lambda(FF^{\top}F)_{i,j}
\end{equation}
After this, we can construct the additive update rule for minimisation:
\begin{equation}\label{update-rule}
F_{i,j}\leftarrow F_{i,j}-\zeta_{i,j}\frac{\partial g}{\partial F_{i,j}}
\end{equation}
where $\zeta_{i,j}$ is a positive step size. In order to keep $F_{i,j}$ staying nonnegative, $\zeta_{i,j}$ can be selected as:
\begin{equation}\label{step-size}
\zeta_{i,j}=\frac{F_{i,j}}{(FF^{\top}BB^{\top}F)_{i,j}+(BB^{\top}FF^{\top}F)_{i,j}+2\lambda(FF^{\top}F)_{i,j}}
\end{equation}
Substituting Eq. (\ref{step-size}) into Eq. (\ref{update-rule}), we have:
\begin{equation}
F_{i,j}\leftarrow F_{i,j}\frac{2(BB^{\top}F)_{i,j}}{(FF^{\top}BB^{\top}F)_{i,j}+(BB^{\top}FF^{\top}F)_{i,j}+2\lambda(FF^{\top}F)_{i,j}}
\end{equation}
We will refer to this method as P-NMF  belief compression in the remainder of this paper.

\section{Experimental Results}

In previous literature, belief compression is usually conceptually demonstrated on small-scale benchmark problems. Here, we are interested in seeing its actual performance on POMDPs with large numbers of states.
We use empirical methods to investigate the following two questions:
\begin{enumerate}
\item whether P-NMF could improve over O-NMF and LP-NMF in policy quality as expected, since it is more focused on error reduction in the value function; 
\item whether belief compression in general is a preferable technique to state-of-the-art POMDP solvers, or under what situation it is so. 
\end{enumerate}

\subsection{Experiments on Benchmark Problems}

The P-NMF, O-NMF and LP-NMF belief compression methods are now compared on four POMDP problems, including UnderwaterNavigation \citep{Kurniawati08}, LifeSurvey \citep{Smith07}, RockSample[7,8] \citep{Smith04}, and a spoken dialogue system problem (ComplexGoalDialog) that is an updated version
\footnote{The only update here is to add an extra initial state $s_0$ standing for the beginning of a dialogue, for which the transition probabilities to the other states $s$ for all system actions are set to be proportional to $e^{-|s|}$ with self-transition eliminated, where $|s|$ stands for the number of user goals contained in $s$. The purpose of such an modification is to yield more natural conversations.} 
of the complex user goal dialogue problem described in \citep{Crook11}. 

Perseus \citep{Spaan05} is employed to solve the compressed POMDPs. In addition, SARSOP \citep{Kurniawati08} is used to give a measure of the rewards that can be achieved for uncompressed problems, as Perseus tends to be unable to solve a POMDP with more than a few thousand states in acceptable time. SARSOP is a state-of-the-art PBVI POMDP solver that achieves its efficiency by restricting itself to only sample belief points near the subset of those reachable from the initial belief under optimal sequences of actions, where the sampling region is controlled by heuristic exploration of sampling paths. Hence, its performance here also gives us an insight into the practical competitiveness of belief compression in solving large-scale POMDPs.

In the following experiments, the P-NMF, O-NMF and LP-NMF belief compression algorithms are implemented in Matlab. The parameters for P-NMF and LP-NMF are empirically tuned based on preliminary experiments. For O-NMF, we follow \citeauthor{Li07}'s (\citeyear{Li07}) method of automatically selecting $\lambda$ to enforce the orthogonal constraint (though the orthogonality is unachievable).  The compression and policy optimisation processes are based on 20,000 belief points randomly sampled by Perseus, in all the tasks except RockSample[7,8] where 100,000 beliefs are sampled to maintain proper performance of the belief compression methods. The computing time for all algorithms is measured using CPU time (seconds) on a computer with 2$\times$Six-Cores 2.4GHz CPUs and 128GB memory. The compression levels $k$ are selected empirically in consideration of both compression errors and computational complexities. 
We sample 1000 decision trajectories to compute an average reward for each policy obtained in each task. The policy learning and reward sampling procedures are repeated five times for each task to calculate a mean and standard deviation, as summarised in Table \ref{tab:benchmark}. 
\begin{table}[t]
\begin{tabular}{lr|cccc}
\hline
Task & Algorithm  & $k$ & Reward & Time (P) & Time (C)  \\
\hline
\hline
ComplexGoalDialog & P-NMF &100&23.5$\pm$0.3&500 & $7.3\times10^3$\\ 
\raisebox{3pt}{\tiny($|\mathcal{S}|:4097$, $|\mathcal{A}|:23$, $|\mathcal{Z}|:49$)}& O-NMF &100&23.1$\pm$0.2 &500& $1.6\times10^3$\\
& LP-NMF &100&13.2$\pm$0.3& 500& $4.2\times10^3$ \\
& SARSOP &n/a&24.2$\pm$0.3&600 & n/a\\
\hline
LifeSurvey & P-NMF &200&89.9$\pm$5.8 &1,500 & $1.9\times10^4$\\
\raisebox{3pt}{\tiny($|\mathcal{S}|:7841$, $|\mathcal{A}|:7$, $|\mathcal{Z}|:28$)}& O-NMF &200&73.3$\pm$12.9&600 & $6.5\times10^3$\\
& LP-NMF &200&$(-8.1\pm1.2)\times10^3$ & 150& $1.4\times10^4$ \\
& SARSOP &n/a&105.7$\pm$1.1&200& n/a\\
\hline
UnderwaterNavigation & P-NMF &100&530.3$\pm$37.8&180 & $2.9\times10^3$\\
\raisebox{3pt}{\tiny($|\mathcal{S}|:2653$, $|\mathcal{A}|:6$, $|\mathcal{Z}|:103$)}& O-NMF &100&354.5$\pm$31.4&300&$1.0\times10^3$\\
& LP-NMF &100&489.7$\pm$283.2&300& $0.9\times10^3$\\
& SARSOP &n/a&733.5$\pm$5.0&160 & n/a\\
\hline
RockSample[7,8] & P-NMF &100&13.2$\pm$0.3&6,000&$2.9\times10^5$\\
\raisebox{3pt}{\tiny($|\mathcal{S}|:12545$, $|\mathcal{A}|:13$, $|\mathcal{Z}|:2$)}& O-NMF &100&7.4$\pm$0.0&60&$4.4\times10^4$\\
& LP-NMF &100&7.4$\pm$0.0& 70 & $3.0\times10^3$\\
& SARSOP &n/a&20.9$\pm$0.5&130&n/a\\
\hline
\end{tabular}
\caption{Performance of belief compression on benchmark problems. Time (P) and Time (C) stand for policy optimisation time and compression time (averaged in five trials), respectively.}\label{tab:benchmark}
\end{table}

It can be found that the proposed P-NMF method consistently outperforms O-NMF and LP-NMF with respect to the rewards obtained in all the four tasks. The performance of LP-NMF compares unfavourably to that of the other two belief compression methods in most of the tasks, and is sometimes very unstable, e.g.\ large variances in its average sampled rewards can be observed in the UnderwaterNavigation problem. Note that, for Rocksample[7,8] one can regard O-NMF and LP-NMF as unable to solve the problem properly, as only 1 $\alpha$-vector is produced in their policies.

Nevertheless, none of the above belief compression methods can guarantee to work on all POMDPs. For example, we also apply them to another two benchmark problems, Homecare \citep{Kurniawati08} and Fourth CIT \citep{Cassandra97}, where they all fail to produce a meaningful policy. The failures could partially due to the limitations of Perseus, e.g.\ for Fourth CIT ($|\mathcal{S}|=1052$, $|\mathcal{A}|=4$,  $|\mathcal{Z}|=28$), Perseus itself fails to solve it, even though 100,000 belief points are sampled. Another reason could be the lack of low-rankness in the reachable belief space of a problem, e.g.\ for Homecare ($|\mathcal{S}|=5408$, $|\mathcal{A}|=9$,  $|\mathcal{Z}|=928$), if we sample a $B$ of 10,000 belief points and analyse its singular values, it can be found that there are about 2,000 of them with non-trivial values. Therefore, a too low-dimensional compression cannot sufficiently approximate the problem, whilst a compression with too many dimensions is computationally intractable for both Perseus and the compression algorithm itself. 

Moreover, considering either time expense or policy quality, belief compression in general tends to be less competitive than SARSOP, which motivates further investigation in the next section.

\subsection{Belief Compression vs.\ the State-of-the-art: A Further Comparison}

Although SARSOP \citep{Kurniawati08} demonstrates its impressive efficiency and effectiveness in the above experiments, it can usually be noticed that its intermediate belief sampling procedure (a heuristic search-based sampling strategy) results in considerably higher space complexities than standard PBVI algorithms (such as Perseus). This situation tends to be worse for problems with more actions and observations. 
Hence, it suggests a potential advantage for belief compression, as SARSOP may easily run out of memory for some large-scale problems. 
To further explore the relative merits of SARSOP and belief compression methods, we investigate two further dialogue problems as follows. 

\subsubsection{POMDP Construction}

Firstly, a hand-crafted rule-based spoken dialogue system was set up in a laboratory environment. The system's domain is restaurant search and it supports complex user goals \citep{Crook10}. Each simple goal (which forms part of a complex user goal) contains two pieces of information (2 attributes: food-type and location, with 4 and 3 values respectively). The combination of the attribute values forms $4\times3=12$ distinct goals. Each state of the system is drawn from the power-set of the 12 possible goals, i.e.\ $2^{12}=4096$ states in total.
We define 64 system actions and 807 user observations, with respect to different expressions of the state information. After this, 161 dialogues were collected from volunteers, and were manually transcribed and annotated. Then we used the collected data to train a POMDP. As the rule-based system assumes the user goal to be unchangeable during a dialogue, we preserve this setting in the first version of the POMDP  (called ``Dialog/Identity'') by setting the transition matrices to the identity matrix for all system actions, except those for the initial state $s_0$. (For the POMDP model trained here we take the state with an empty goal set as $s_0$.) For $s_0$ the transition probabilities to other states for all actions are defined to be proportional to the probability mass of a Poisson distribution over the number of goals contained in the target state. A non-zero reward is only given to those actions that present goal information to the user, where $R(s,a)=10$ if the information presented by $a$ is fully contained in $s$, $R(s,a)=-10$ otherwise. The observation probabilities are modelled using the exponential family, for which the parameter is trained by maximising the regularised log-likelihood on the data. To be concrete, we use feature vectors to represent the states, actions and observations, as $\phi_s(s)$,  $\phi_a(a)$ and $\phi_z(z)$ respectively, where binary values are used to indicate the occurrences of certain attribute values, and extra fields are introduced for actions and observations indexing the dialogue act type. Then we let the joint feature representation of a $(z,s,a)$-tuple be the tensor product of the corresponding individual feature vectors as $\phi(z,s,a)=\phi_s(s)\otimes\phi_a(a)\otimes\phi_z(z)$, and formulate the observation probabilities as 
\begin{equation} \Omega(z|s,a)=\frac{\exp[w^{\top}\phi(z,s,a)]}{\sum_{z'}\exp[w^{\top}\phi(z',s,a)]}
\end{equation}
 where $w$ is the parameter to be trained on the data. Finally, we eliminate those observation probabilities below the threshold $10^{-6}$ and re-normalise the distributions, to achieve a sparser problem for space efficiency purposes. 
\begin{table}[t]
\begin{tabular}{lr|cccc}
\hline
Task & Algorithm &  $k$ & Reward & Time (P) & Time (C) \\
\hline
\hline
Dialog/Identity & P-NMF &100&41.3&3,000& $3.0\times10^4$\\
\raisebox{3pt}{\tiny($|\mathcal{S}|:4096$, $|\mathcal{A}|:64$, $|\mathcal{Z}|:807$)}& SARSOP &n/a&49.0&2,712&n/a\\
\hline
Dialog/Browsing & P-NMF &100&23.1&3,000 & $3.0\times10^4$\\
\raisebox{3pt}{\tiny($|\mathcal{S}|:4096$, $|\mathcal{A}|:64$, $|\mathcal{Z}|:807$)}& SARSOP &n/a &--&$\infty$&n/a\\
\hline

\end{tabular}
\caption{SARSOP vs\. belief compression on two challenging dialogue problems. P-NMF compressed problems are solved using Perseus, with 100,000 sampled belief points for compression and policy optimisation.}\label{tab:dialog}
\end{table}

\subsubsection{Results}

The performance of SARSOP and P-NMF belief compression for this problem is compared in the upper half of Table \ref{tab:dialog} (Dialog/Identity), where SARSOP runs out of memory\footnote{Here we set a 100GB memory limit for SARSOP to reserve sufficient resource for normal operating system activities.} within its first 30 iterations (after the initialisation step), but a promising policy is still obtained. A possible reason for this could be the identity transition matrices, which make the problem converge easily. Therefore, we slightly modified the previous setting to provide a further realistic challenge, in the problem ``Dialog/Browsing''. In this second version of the POMDP, instead of keeping the user goal fixed, we assume that a user will be able to request for alternative goals when its current goal is correctly presented by the system, and this process can last for an infinite number of turns (i.e.\ an infinite-horizon planning problem). This corresponds to dialogues where users are browsing through the possible entities (e.g.\ find out what Thai restuarants there are, and then search for the closest restaurant).

To enable such a setting, we re-define those transition probabilities $T(\cdot|a,s)$ that have $R(s,a)=10$ to be identical to $T(\cdot|a,s_0)$. The performance of SARSOP and the P-NMF method on this modified POMDP (Dialog/Browsing) is shown in the lower half of Table \ref{tab:dialog}. This time, SARSOP fails to finish in acceptable time, as it takes more than 15 hours to run each iteration when initialising the fast informed bound \citep{Hauskrech00}. On the contrary, the efficiency of P-NMF is much more preferable in this case. Furthermore, by looking into the sampled dialogue trajectories of the compressed problem, we found that the ratio between its correct decisions (reward 10) and incorrect decisions (reward -10) is approximately 3:1, which suggests a reasonable quality of the policy.

\section{Conclusion}
This paper introduces a theoretical framework to analyse linear belief compression techniques, under which the deficiencies of three existing algorithms are presented. The findings indicate that policy quality reduction resulting from the compression can be relieved if those deficiencies are properly revised, as demonstrated by a new proposed P-NMF model. However, the overall performance of belief compression techniques tends to be less competitive in comparison with a state-of-the-art POMDP solver, such as SARSOP. However, we show that under particular situations SARSOP may fail due to time or space complexities, whilst belief compression could provide a feasible solution. A further question posed here would be whether there exists a way of combining SARSOP and belief compression to achieve further efficiencies. Unfortunately, our preliminary answer is negative, as various underlying theories (e.g.\ the fast informed bound) in SARSOP's heuristics rely on beliefs being distributions. The possibility of applying alternative heuristics in solving compressed problems still requires further investigation.
\bibliography{my}{}
\bibliographystyle{spbasic}

\end{document}